\documentclass{article}
\PassOptionsToPackage{numbers, compress}{natbib}
\usepackage[final]{neurips_2022}

\usepackage[utf8]{inputenc} % allow utf-8 input
\usepackage[T1]{fontenc}    % use 8-bit T1 fonts
\usepackage{hyperref}       % hyperlinks
\usepackage{url}            % simple URL typesetting
\usepackage{booktabs}       % professional-quality tables
\usepackage{amsfonts}       % blackboard math symbols
\usepackage{nicefrac}       % compact symbols for 1/2, etc.
\usepackage{microtype}      % microtypography
\usepackage{xcolor}         % colors
\usepackage{natbib}
\usepackage{multirow}
\usepackage{amssymb}% http://ctan.org/pkg/amssymb
\usepackage{pifont}% http://ctan.org/pkg/pifont

\usepackage{microtype}
\usepackage{graphicx}
\usepackage{subfigure}
\usepackage{bbm}
\usepackage{booktabs} % for professional tables
\usepackage{amsmath,amsfonts,amsthm,amssymb}
\usepackage{mathrsfs}

\newtheorem{proposition}{Proposition}

\usepackage{setspace}
\usepackage{graphicx}
\usepackage{subfigure}
\usepackage{hyperref}
\usepackage{xcolor}
\newcount\comments  % 0 suppresses notes to selves in text
\newcommand{\genComment}[2]{\ifnum\comments=1{\textcolor{#1}{\textsf{\footnotesize #2}}}\fi}

\usepackage{wrapfig}
\usepackage{setspace}
\usepackage{graphicx}
\usepackage{subfigure}
\usepackage{booktabs} % for professional tables

\usepackage{hyperref}
\usepackage{algorithmic}
\usepackage{algorithm}

\usepackage{microtype}
\usepackage{graphicx}
\usepackage{subfigure}
\usepackage{booktabs} % for professional tables
\usepackage{setspace}
\usepackage{graphicx}
\usepackage{subfigure}
\usepackage{booktabs} % for professional tables
\usepackage{microtype}
\usepackage{graphicx}
\usepackage{subfigure}
\usepackage{bm}
\usepackage{amsmath,amsfonts,amsthm,amssymb}
\usepackage{mathrsfs}

\usepackage{hyperref}

\title{Toward Causal-Aware RL: \\ State-Wise Action-Refined Temporal Difference}

\author{%
  Hao Sun \\ \texttt{sunhopht@gmail.com} \And Taiyi Wang \\ \texttt{Taiyi.Wang@cl.cam.ac.uk} \\
}
\begin{document}

\maketitle

\begin{abstract}
Although it is well known that exploration plays a key role in Reinforcement Learning (RL), prevailing exploration strategies for continuous control tasks in RL are mainly based on naive isotropic Gaussian noise regardless of the causality relationship between action space and the task and consider all dimensions of actions equally important. In this work, we propose to conduct interventions on the primal action space to discover the causal relationship between the action space and the task reward. We propose the method of State-Wise Action Refined (SWAR), which addresses the issue of action space redundancy and promote causality discovery in RL. We formulate causality discovery in RL tasks as a state-dependent action space selection problem and propose two practical algorithms as solutions. The first approach, TD-SWAR, detects task-related actions during temporal difference learning, while the second approach, Dyn-SWAR, reveals important actions through dynamic model prediction. Empirically, both methods provide approaches to understand the decisions made by RL agents and improve learning efficiency in action-redundant tasks. %~\footnote{Code is made publicly available at https://github.com/2Groza/Action-Refined-Temporal-Difference.}%~\footnote{To accelerate training and leverage dynamical computational graph in RL, the source code of INVASE is re-implemented with PyTorch~\cite{paszke2017automatic}.}

% In this work, I first propose Fractal Curriculum INVASE to improve the training efficiency as well as the asymptotic performance of the vanilla INVASE. Then the improved model is used for action redundancy elimination in reinforcement learning. 

\end{abstract}
\section{Introduction}
\label{intro}
Although model-free RL has achieved great success in various challenging tasks and outperforms experts in most cases~\cite{mnih2015human,silver2016mastering,lillicrap2015continuous,vinyals2019grandmaster,berner2019dota}, the design of action space always requires elaboration. For example, in the game StarCraftII, hundreds of units can be selected and controlled to perform various actions. To tackle the difficulty in exploration caused by the extremely large action and state space, hierarchical action space design and imitation learning are used~\cite{sun2018tstarbots,vinyals2019grandmaster} to reduce the exploration space. Both of those approaches require expert knowledge of the task. On the other hand, even in the context of imitation learning where expert data is assumed to be accessible, causal confusion will still hinder the performance of an agent~\cite{de2019causal}. Those defects motivate us to explore the causality-awareness of an agent that permits an agent to discover the causal relationship for the environment and select useful dimensions of action space during policy learning in pursuance of improved learning efficiency.
Another motivating example is the in-hand manipulation tasks~\cite{andrychowicz2020learning}: robotics equipped with touch sensors outperforms the policies learned without sensors by a clear margin in hand-in manipulation tasks~\cite{melniktactile}, showing the importance of causality discovery between actions and feedbacks in RL. 
A similar example can be found in human learning: knowing nothing about how to control the finger joints flexibly may not hinder a baby learns to walk, and a baby has not learned how to walk can still learn to use forks and spoons skillfully, inspiring us to believe that the challenge for exploration can be greatly eased after the causality between action space and the given task is learned.

In this work, the recent advance of instance-wise feature selection technique~\cite{yoon2018invase} is improved to be more suitable in large-scale state-wise action selection tasks and adapted to the time-series causal discovery setting to
select state-conditioned action space in RL with redundant action space. With the proposed method, the agent learns to perform intervention, discover the true structural causal model (SCM) and select task-related actions for a given task, remarkably reduces the burden of exploration and obtains on-par learning efficiency as well as asymptotic performance compared with agents trained in the oracle settings where the action spaces are pruned according to given tasks manually.

%we are motivated to solve the challenge of finding the optimal action space automatically.

\section{Preliminary}

\begin{figure*}[t]
\centering
\begin{minipage}[b]{1.\linewidth}
\includegraphics[width=1.0\linewidth]{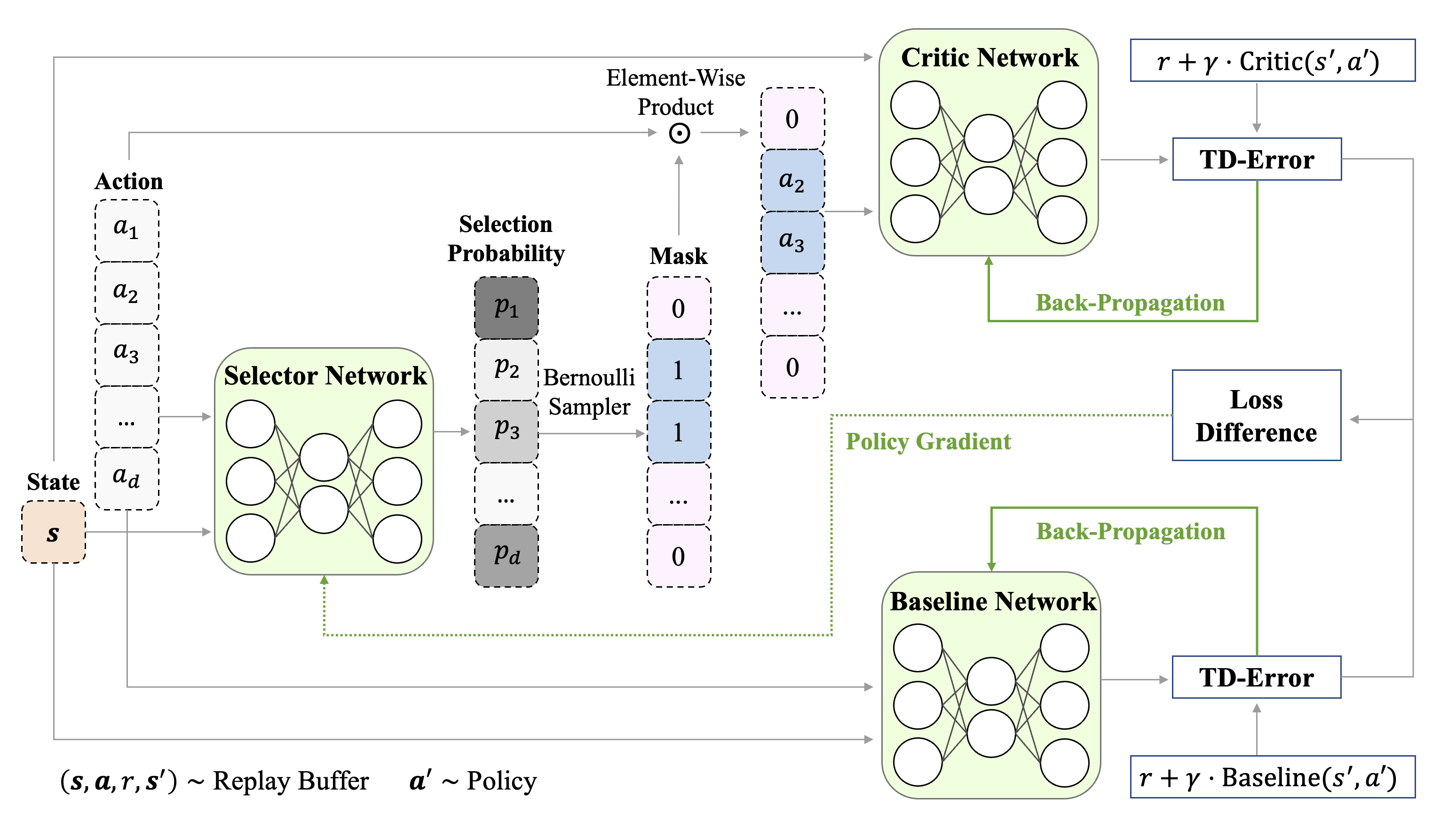}
\end{minipage}%
\caption{\textbf{Block diagram of INVASE in temporal difference learning}. States and actions sampled from replay buffer are fed into the selector network that predicts the selection probabilities of different dimensions of actions. A selection mask is then generated according to such a selection probability vector. The critic network and the baseline network are trained to minimize temporal difference error with states and the selected dimension of actions and primal action respectively. The difference of TD-Error is used to conduct a policy gradient to update the selector network.}
\label{teaser}
\end{figure*}
\paragraph{Markov Decision Processes}
RL tasks can be formally defined as Markov Decision Processes (MDPs), where an agent interacts with the environment and learns to make decision at every timestep. Formally, we consider the deterministic MDP with a fixed horizon $H\in\mathbb{N}^+$ denoted by a tuple $(\mathcal{S},\mathcal{A}, H,r,\gamma,\mathcal{T},\rho_0)$, where $\mathcal{S}$ and $\mathcal{A}$ are the $|\mathcal{S}|$-dimensional state and $|\mathcal{A}|$-dimensional action space; $r:\mathcal{S}\times\mathcal{A} \mapsto \mathbb{R}$ denotes the reward function; $\gamma \in (0,1]$ is the discount factor indicating importance of present returns compared with long-term returns;
$\mathcal{T}: \mathcal{S} \times \mathcal{A} \mapsto \mathcal{S}$ denotes the transition dynamics; $\rho_0$ is the initial state distribution.
 
We use $\Pi$ to represent the stationary deterministic policy class, i.e., $\Pi = \{\pi: \mathcal{S}\mapsto\mathcal{A}\}$.
The learning objective of an RL algorithm is to find $\pi^*\in\Pi$ as the solution of the following optimization problem: 
$\max_{\pi\in\Pi} \mathbb{E}_{\tau\sim\rho_0, \pi, \mathcal{T}}[\sum_{t=1}^H \gamma^t r_t]$
where the expectation is taken over the trajectory $\tau = (s_1, a_1, r_1, \dots, s_H, a_H, r_H)$ generated by policy $\pi$ under the environment $\mathcal{T}$, starting from $s_0\sim \rho_0$.

\paragraph{INVASE}
INVASE is proposed by~\cite{yoon2018invase} to perform instance-wise feature selection to reduce overfitting in predictive models. The learning objective is to minimize the KL-Divergence of the full-conditional distribution and the minimal-selected-features-only conditional distribution of the outcome, i.e., $\min_{F} \mathcal{L}$, with 
\begin{equation}
\label{eq_invase}
\begin{array}{ll}
    \mathcal{L} = \mathcal{D}_{KL}(p(Y|X = x )||p( Y |X^{(F(x))} = x^{(F(x)}))) + \lambda |F(x)|_0.
\end{array}
\end{equation}
where $F: \mathcal{X} \to \{0,1\}^d$ is a feature selection function and $|F(x)|_0$ denotes the cardinality ($l_0$ norm) of selected features, i.e., the number of $1$'s in $F(x)$.~\footnote{To avoid confusion between state notion $s\in\mathcal{S}$ and the selector notion $S$ used in~\cite{yoon2018invase}, $F$ is used in this work to represent the selector (i.e., mask generator).} $d$ is the dimension of input features. $x^{(F(x))} = F(x) \odot x$ denotes the element-wise product of $x$ and generated mask $m = F(x)$. Ideally, the optimal selection function $F$ should be able to minimize the two terms in Equation (\ref{eq_invase}) simultaneously.

INVASE applies the Actor-Critic framework in the optimization of $F$ through sampling, where $f_\theta(\cdot|x)$, parameterized by a neural network $\theta$~\footnote{In this work, subscripts $\phi, \psi, \theta, w$ are used to denote the parameter of neural networks.}, is used as a stochastic actor. Two predictive networks $C_{\phi}(\cdot), B_{\psi}(\cdot)$ are considered as the critic and the baseline network used for variance reduction~\cite{williams1992simple} and trained with the Cross-Entropy loss to produce return signal $\mathcal{L}$, based on which $f_\theta(\cdot|x)$ can be optimized through policy gradient:
\begin{equation}
    \mathbb{E}_{(x,y)\sim p}[\mathbb{E}_{m\sim f_\theta(\cdot|x)}[\mathcal{L} \nabla_{\theta}\log f_\theta (\cdot|x)] ].
\end{equation}
Finally, $F(x) = (F_1(x),...,F_d(x))$ can be get by sampling from $f(\cdot|x) = (f_1(x),...,f_d(x))$, with 
\begin{equation}
	F_i(x) = \begin{cases}
	1, \quad\mathbf{w.p.} \quad f_i(\cdot|x).\\
	0, \quad\mathbf{w.p.} \quad 1 - f_i(\cdot|x).
		   \end{cases}%
\end{equation}

\section{Proposed Method}

The objective of this work is to carry out state-wise action selection in RL through intervention, and thereby enhance the learning efficiency with a pruned task-related action space after finding the correct causal model. 
Section~\ref{method_motivation} starts with the formalization of the action space refinery objective in RL tasks under the framework of causal discovery. %where redundant dimensions of actions exist. 
Section~\ref{method_icinvase} introduces SWAR, which improves the scalability of INVASE in high dimensional variable selection tasks. We integrate SWAR with deterministic policy gradient methods~\cite{silver2014deterministic} in Section~\ref{method_actionspacepruning} to perform state-wise action space pruning, resulting in two practical causality-aware RL algorithms.

% \subsection{Structural Causal Model for Reinforcement Learning}
% \label{method_scm}

\subsection{Temporal Difference Objective with Structural Causal Models}
\label{method_motivation}
In modern RL algorithms, the most general approach is based on the Actor-Critic framework~\cite{konda2000actor}, where the critic $Q_w(s,a)$ approximates the return of given state-action pair $(s,a)$ and guides the Actor to maximize the approximated return at state $s$. The Critic is optimized to reduce Temporal Difference (TD) error~\cite{sutton1998reinforcement}, defined as
\begin{equation}
    \mathcal{L}_{TD} = \mathbb{E}_{s_i,a_i,r_i, s'_i\sim\mathcal{B}} [(r_i + \gamma Q_w(s'_i, a'_i) -Q_w(s_i, a_i) )^2].
\end{equation}
where $\mathcal{B} = {(s_i, a_i, r_i, s'_i)}_{i=1,2,...}$ is the replay buffer used for off-policy learning~\cite{lillicrap2015continuous,fujimoto2018addressing,haarnoja2018soft}, and $a'_i = \pi(s'_i)$ is the predicted action for state $s'_i$. In practice, the calculations of $Q_w(s'_i,a'_i)$ are usually based on another set of slowly updated target networks for stability~\cite{fujimoto2018addressing,haarnoja2018soft}. Henceforth, TD-learning can be roughly simplified as regression with notion $y_i = r_i + \gamma Q_w(s'_i, a'_i)$: 
\begin{equation}
    \mathcal{L}_{TD} = \mathbb{E}_{s_i,a_i,r_i, s'_i\sim\mathcal{B}} [(y_i -Q_w(s_i, a_i) )^2].
\end{equation}
Assume there are only $M<L$ actions are related to a specific task among the $L$-dimensional actions $a_i = a_i^{(1)},...,a_i^{(L)}$, i.e., $Q_w(\cdot, \cdot)$ is function of $s_i, a_i^{(1)},...,a_i^{(M)}$. Learning with the primal redundant action space will lead to around $\frac{L+|\mathcal{S}|}{M+|\mathcal{S}|}$ times sample complexity~\cite{even2006action,zahavy2018learn}. 
Therefore, we are motivated to improve the learning efficiency of $Q$ by pruning those task-irrelevant action dimensions $a_i^{(M+1)},...,a_i^{(L)}$ by finding an action selection function $G$, satisfying
\begin{equation}
\label{eq_td_invase}
    \min_{G, Q_w} \mathbb{E}_{s_i,a_i,r_i, s'_i\sim\mathcal{B}} [(y'_i -Q_w(s_i, a_i^{(G(a_i|s_i))}) )^2]  + \lambda |G(a_i|s_i)|_0.
\end{equation}
where $y'_i =  r_i + \gamma Q_w(s'_i, a_i^{'G(a'_i|s_i)})$.% for consistency.

Such a problem can be addressed from the perspective of causal discovery. Formally, we can use the Structural Causal Models (SCMs) to represent the underlying causal structure of a sequential decision making process, as shown in Figure~\ref{fig:scms}. Under this language, we use the notion of \textbf{causal} actions to denote $a_i^{(1,...,M)}$, and \textbf{nuisance} actions for other dimension of actions. In our work, we use IC-INVASE for causal discovery. Ideally, the action selection function $G$ should be able to distinguish between nuisance action dimensions and the causal ones that has causal relation with either dynamics or reward mechanism. We present in the next section our causal discovery algorithms.

\begin{figure}[ht]
    \centering
    \includegraphics[width=0.8\linewidth]{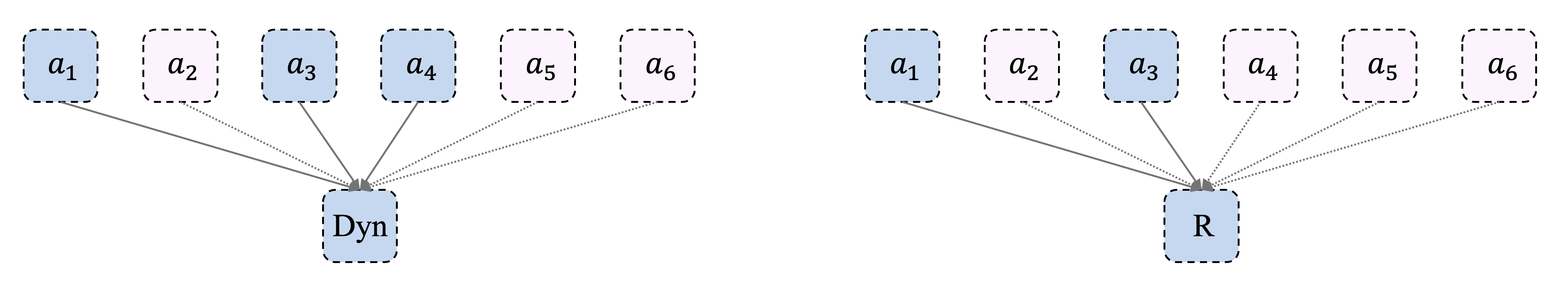}
    \caption{SCM of temporal difference learning. Among all executable actions, there can be only a subset have effect on the dynamical changes or the reward mechanism. In our work, we use IC-INVASE as a causal discovery tool to distinguish the causal irrelevant actions and hence improve learning efficiency.}
    \label{fig:scms}
\end{figure}

% The underlying sequential decision process has complex causal structures, represented in Fig 2. States influence future expert actions, and are also themselves influenced by past actions and states.

%Action Elimination ~\cite{even2006action}
%In this section, we demonstrate the redundant actions in RL tasks lead to additional variance in policy gradient methods, therefore hinder the learning efficiency and even the asymptotic performance~\cite{henderson2018deep}.

\subsection{Iterative Curriculum INVASE (IC-INVASE)}
\label{method_icinvase}
Instead of directly applying INVASE to solve Equation (\ref{eq_td_invase}). We first propose two improvements to make the vanilla INVASE more suitable for large-scale variable selection tasks as the action dimension in RL might be extremely large~\cite{vinyals2019grandmaster}. Specifically, the first improvement, based on curriculum learning, is introduced to tackle the exploration difficulty when $\lambda$ in Equation (\ref{eq_invase}) is large, where INVASE tends to converge to poor sub-optimal solutions and prune all variables including the useful ones~\cite{yoon2018invase}. The second improvement is based on the iterative structure of variable selection tasks: the feature selection operator $G$ can be applied multiple times to conduct hierarchical feature selection without introducing extra computation expenses.
\subsubsection{Curriculum Learning For High Dimensional Variable Selection}
The work of~\cite{bengio2009curriculum} first introduces Curriculum Learning to mimic human learning by gradually learn more complex concepts or handle more difficult tasks. Effectiveness of the method has been demonstrated in various set-ups~\cite{bengio2009curriculum,matiisen2019teacher,czarnecki2018mix,weinshall2018curriculum,xu2020curriculum}. In general, it should be easier to select $M$ useful variables out of $L$ input variables when $M$ is larger. The most trivial case is to select all $L$ variables, with an identical mapping $x^{(G(x))} = G(x)\odot x = x$. Formally, we have

\begin{proposition}[Curriculum Property in Variable Selection]
\label{prop_1}
Assume $M$ out of $L$ variables are outcome-related, let $M\le N_1<N_2 \le L$, $G_{N_1}(x)$ minimizes $\mathcal{D}_{KL}(p(Y|X = x )||p( Y |X^{(G(x))} = x^{(G(x))})) + \lambda ||G(x)|_0-N_1|$. Then  \\  $G_{N_2}(x)$ minimizes $\mathcal{D}_{KL}(p(Y|X = x )||p(Y|X^{G(x)} = x^{G(x)}))+ \lambda ||G(x)|_0-N_2|$ can be get through: \\
$G_{N_2}(x) \in \{G_{N_1}(x) \lor  [G_{N_1}(x) \mathbf{XOR} ~\mathbf{1}]_{1_{N_2-N_1}}\}$, \\where $[\cdot]_{1_{N_2-N_1}}$ means keep $N_2 - N_1$ none-zero elements unchanged while replacing other elements by $0$.
\end{proposition}
\begin{proof}
By the definition of the $[\cdot]_{1_{N_2-N_1}}$ operator, $||G(x)|_0-N_2|=0$ is minimized. On the other hand, starting from $N_1 = M$, minimizing $\mathcal{D}_{KL}(p(Y|X = x )||p(Y|X^{(G(x))} = x^{(G(x))}))$ requires all the $M$ outcome-related variables being selected by $G_{N_1}$. Therefore, $G_{N_2}$ also minimizes the KL-divergence by the independent assumption of the other $L-M$ variables with the outcomes.
\end{proof}
The proposition indicates the difficulty of selecting $N$ useful out of $L$ variables decreases monotonically as $N \ge M$ increase from $M,M+1, ...,L$. In this work, two classes of practical curriculum are designed: 1. curriculum on the $l_0$ penalty coefficient, and 2. curriculum on the proportion of variables to be selected.
\paragraph{Curriculum on $l_0$ Penalty Coefficient}
In this curriculum design, the penalty coefficient $\lambda$ in Equation (\ref{eq_invase}) is increased from $0$ to a pre-determined number (e.g., $1.0$). Increasing the value of $\lambda$ will lead to a larger penalty on the number of variables selected by the feature selector. Experiments in ~\cite{yoon2018invase} has shown a large $\lambda$ always lead to a trivial selector that does not select any variable.
\paragraph{Curriculum on the Proportion of Selected Features}
In this curriculum design, the proportion of variables to be selected, denoted by $p_r$, is adjusted from the default setting $0$ to a decreasing number from a pre-determined value (e.g., $0.5$) to $0$.
i.e., the $l_0$ penalty term $\lambda|G(x)|_0$ in Equation (\ref{eq_invase}) is revised to be $\lambda||G(x)|_0 - d\cdot p_r|$,
where $d$ is the dimension of input $x$. When the proportion is set to be $p_r = 0.5$, the selector will be penalized whenever less or more than half of all variables are selected. Such a curriculum design forces the feature selector to learn to select less but increasingly more important variables gradually.

Thus, we get the learning objective of curriculum-INVASE:
\begin{equation}
\label{eq_invase_cl}
    \mathcal{L}= \mathcal{D}_{KL}(p(Y|X = x )||p( Y |X^{(G(x))} = x^{(G(x)}))) + \lambda||G(x)|_0 - d\cdot p_r|.
\end{equation}
where $\lambda$ increases from $0$ to some value and $p_r$ decreases from a value in $[0,1)$ to $0$.

\subsubsection{Iterative Variable Selection}
The second improvement proposed in this work is based on the iterative structure of variable selection tasks. Specifically, the $G(x)$ mapping $x\in \mathcal{X}$ to $\{0,1\}^d$ is an iterative operator, which can be applied for multiple times to perform coarse-to-fine variable selection. Although in practice we follow ~\cite{yoon2018invase} to apply an element-wise product in producing $x^{(G(x))}$: $x^{(G(x))} = G(x)\odot x \in \mathcal{X}$. In more general cases, the i-th element of $x_i^{(G(x))}$ is
\begin{equation}
	x_i^{(G(x))} = \begin{cases}
	1, \quad\mathbf{if} \quad G_i(x)=1.\\
	*, \quad\mathbf{if} \quad G_i(x)=0.
		   \end{cases}%
\end{equation}
where $*$ can be an arbitrary identifiable indicator that represents the variable is not selected.
\begin{algorithm}[t]
\caption{TD3 with TD-SWAR}
\label{Algorithm1}
\begin{algorithmic}
% 	\STATE \textbf{Require}
		\STATE Initialize critic networks $C_{\phi_1}$, $C_{\phi_2}$, baseline networks $B_{\psi_1}$, $B_{\psi_2}$ and actor network $\pi_{\nu}$, IC-INVASE selector network $G_{\theta}$
		\STATE Initialize target networks $\phi'_1 \leftarrow {\phi_1}$, $\phi'_2 \leftarrow {\phi_2}$, $\psi'_1 \leftarrow {\psi_1}$, $\psi'_2 \leftarrow {\psi_2}$, $\nu' \leftarrow {\nu}$
		\STATE Initialize replay buffer $\mathcal{B}$
	    \FOR{$t = 1,H$}
	        \STATE Interact with environment and store transition tuple $(s,a,r,s')$ in $\mathcal{B}$
	        \STATE Sample mini-batch of transitions $\{(s,a,r,s')\}$ from $\mathcal{B}$
	        \STATE Calculate perturbed next action by $\tilde{a}\leftarrow \pi_{\nu'}(s') + \epsilon$, $\epsilon$ is sampled from a clipped Gaussian.
	        \STATE Select actions with target selector network \\ ~~$\tilde{a}^{(G(\tilde{a}|s'))} \leftarrow G_{\theta'}(\tilde{a}|s') \odot \tilde{a}$
	        \STATE Calculate target critic value $y_c$ and baseline value $y_b$:
	        \STATE ~~$y_c \leftarrow r + \gamma \min_{i=1,2}C_{\phi'_i}(s',\tilde{a}^{(G(\tilde{a}|s'))})$
	        \STATE ~~$y_b \leftarrow r + \gamma \min_{i=1,2}B_{\psi'_i}(s',\tilde{a})$
	        \STATE Update critics and baselines with selected actions: \\ 
	        ~~$a^{(G(a|s))} \leftarrow G_{\theta}(a|s') \odot a $ \\
	        ~~$\phi_i \leftarrow \arg\min_{\phi_i} \mathbf{MSE}(y_c,C_{\phi_i}(s,a^{(G(a|s))}))$ \\
	        ~~$\psi_i \leftarrow \arg\min_{\psi_i} \mathbf{MSE}(y_b,B_{\psi_i}(s,a))$
		    \STATE Update IC-INVASE selector network by the policy gradient, with learning rate $\eta_1$:\\
		    ~~$\theta \leftarrow \theta - \eta_1(l_b - l_c)\nabla_{\theta}\log G_{\theta}(a|s)$, $l_b$, $l_c$ are MSE \\~~ losses in the previous step.
		    \STATE Update $\nu$ by the deterministic policy gradient, with learning rate $\eta_2$:\\
		    ~~$\nu\leftarrow \nu - \eta_2 \nabla_a C_{\phi_1}(s,a)|_{a=\pi_\nu(s)}\nabla_\nu \pi_\nu (s)$
		    \STATE Update target networks, with $\tau \in (0,1)$: \\
		    ~~$\phi'_i \leftarrow \tau {\phi_i} + (1-\tau)\phi'_i$ \\
		    ~~$\psi'_i \leftarrow \tau {\psi_i} + (1-\tau)\psi'_i$ \\
		    ~~$\nu' \leftarrow \tau {\nu} + (1-\tau)\nu'$
		    
		\ENDFOR
\end{algorithmic}
\end{algorithm}

On the other hand, once the outputs $G(x)$ of the selector have been recorded, $*$ can be replaced by any label-independent variable $G(x)\odot z$, where $z\sim p_z(\cdot)$ is outcome-independent. Then $x^{(G(x))}$ can be regarded as a new sample and be fed into the variable selector, resulting in a hierarchical variable selection process:
\begin{equation}
\label{eq_hier}
\begin{array}{ll}
& x^{(1)} = (G(x) \odot x) \oplus (G(x) \odot z),\\
     & x^{(2)} = (G(x^{(1)}) \odot x^{(1)}) \oplus (G(x^{(1)}) \odot z),\\
     & ... \\
     & x^{(n)} = (G(x^{(n-1)}) \odot x^{(n-1)}) \oplus (G(x^{(n-1)}) \odot z),
\end{array}
\end{equation}
where $z\sim p_z(\cdot) $, and $\oplus$ is the element-wise sum operator. Moreover, if the distribution of irrelevant variable $p_{x}(\cdot)$ is known, applying the variable selection operator obtained from Equation (\ref{eq_invase_cl}) for multiple times with $p_{z}(\cdot) \buildrel d \over = p_{x}(\cdot)$ has the meaning of hierarchical variable selection: after each operation, the most obvious $1-p_r$ irrelevant variables are discarded. e.g., when $p_r = 0.5$, ideally top-$50\%, 25\%, 12.5\%$ most important variables will be selected after the first three selection operations. In this work, a coarse approximation is utilized by selecting $z$ to be $z = 0$ for simplicity. ~\footnote{$p_z(\cdot)$ may be learned through generative models to approximate $p_x(\cdot)$, and Equation (\ref{eq_hier}) can be regarded as a kind of data-augmentation or ensemble method. This idea is left for the future work.}

Combining those two improvements lead to an Iterative Curriculum version of INVASE (IC-INVASE) that addresses the exploration difficulty in high-dimensional variable selection tasks. Curriculum learning helps IC-INVASE to achieve better asymptotic performance, i.e., achieve higher True Positive Rate (TPR) and lower False Discovery Rate (FDR), while iterative application of the selection operator contributes to higher learning efficiency: selectors models with different level of TPR/FDR can be generated on-the-fly.
%Combining those two improvements lead to an Iterative Curriculum version of INVASE (IC-INVASE) that addresses the exploration difficulty in high-dimensional variable selection tasks. Curriculum learning helps IC-INVASE to achieve better asymptotic performance, i.e., higher True Positive Rate (TPR) and lower False Discovery Rate (FDR) in variable selection, while iterative application of the selection operator contributes to higher learning efficiency, i.e., selectors models with different level of TPR/FDR can be generated at once.
\subsection{State-Wise Action Refinery with IC-INVASE}
\label{method_actionspacepruning}
\subsubsection{Temporal Difference State-Wise Action Refinery}
With the techniques introduced in the previous section, higher dimensional variable selection tasks can be better solved, therefore we are ready to use IC-INVASE to solve Equation (\ref{eq_td_invase}). The resulting algorithm is called Temporal Difference State-Wise Action Refinery (TD-SWAR).

In this work, TD3~\cite{fujimoto2018addressing} is used as the basic algorithm we build TD-SWAR up on. In addition to the policy network $\pi_\nu$, double critic networks $C_{\phi_1}$, $C_{\phi_2}$ and their corresponding target networks used in vanilla TD3, TD-SWAR includes an action selector model $G_{\theta}$ and two baseline networks $B_{\psi_1}$, $B_{\psi_2}$ following ~\cite{yoon2018invase} to reduce the variance in policy gradient learning. Pseudo-code for the proposed algorithm is shown in Algorithm~\ref{Algorithm1}. And the block diagram in Figure~\ref{teaser} illustrates how different modules in TD-SWAR updates their parameters.

\subsubsection{Static Approximation: Model-Based Action Selection}
While IC-INVASE can be formally integrated with temporal difference learning, the learning stability is not guaranteed. Different from general regression tasks where the label for every instance is fixed across training, in temporal difference learning, the regression target is closely related to the present critic function $C_{\phi}$, the policy $\pi_\nu$ that generates the transition tuples used for training, and the selector model of IC-INVASE itself. In this section, a static approach is proposed to approximately solve the challenge of instability in TD-SWAR~\footnote{Analysis on the approximation is provided in Appendix~\ref{appd_apprx}}.

Other than applying the IC-INVASE algorithm to solve Equation (\ref{eq_td_invase}), another way of leveraging IC-INVASE in action space pruning is to combine it with the model-based methods~\cite{ha2018world,langlois2019benchmarking,hafner2019dream,janner2019trust}, where a dynamic model $\mathcal{P}:\mathcal{S}\times\mathcal{A}\mapsto \mathcal{S}$ is learned through regression:
\begin{equation}
\label{eq_mbrl}
    \mathcal{P} = \arg \min_{\mathcal{P}} \mathbb{E}_{(s,a,s')\sim \pi,\mathcal{T}}(s' - \mathcal{P}(s,a))^2
\end{equation}
Although the task of precise model-based prediction is in general challenging~\cite{sharma2019dynamics}, in this work, we only adopt model-based prediction in action selection, and the target is action discovery other than precise prediction. As the dynamic models are always static across learning, such an approach can be much more stable than TD-SWAR. We name this method as Dyn-SWAR and present the pseudo-code in Algorithm \ref{Algorithm2}, where we infuse IC-INVASE to Equation (\ref{eq_mbrl}) and get the learning objective:
\begin{equation}
    \min_{G,\mathcal{P}} \mathbb{E}_{(s,a,s')\sim \pi,\mathcal{T}}(s' - \mathcal{P}(s,a^{(G(a|s))}))^2
\end{equation}

\section{Experiment}

\begin{algorithm}[t]
\caption{TD3 with Dyn-SWAR}
\label{Algorithm2}
\begin{algorithmic}
% 	\STATE \textbf{Require}
		\STATE Initialize critic networks $Q_{w_1}$, $Q_{w_2}$, Dynamics critic model $C_{\phi}$, dynamic baseline model $B_{\psi}$, actor network $\pi_{\nu}$, and IC-INVASE selector network $G_{\theta}$
		\STATE Initialize target networks $w'_1 \leftarrow {w_1}$, $w'_2 \leftarrow {w_2}$, $\nu' \leftarrow {\nu}$
		\STATE Initialize replay buffer $\mathcal{B}$
	    \FOR{$t = 1,H$}
	        \STATE Interact with environment and store transition tuple $(s,a,r,s')$ in $\mathcal{B}$
	        \STATE Sample mini-batch of transitions $\{(s,a,r,s')\}$ from $\mathcal{B}$
	        \STATE Update dynamic critics and dynamic baselines with equation (\ref{eq_mbrl}): \\ 
	        ~~$\phi \leftarrow \arg\min_{\phi} \mathbf{MSE}(s',C_{\phi}(s,a^{(G(a|s))}))$ \\
	        ~~$\psi \leftarrow \arg\min_{\psi} \mathbf{MSE}(s',B_{\psi}(s,a))$
		    \STATE Update IC-INVASE selector network by the policy gradient, with learning rate $\eta_1$:\\
		    ~~$\theta \leftarrow \theta - \eta_1(l_b - l_c)\nabla_{\theta}\log G_{\theta}(a|s)$, $l_b$, $l_c$ are MSE \\~~ losses in the previous step.
	        \STATE Calculate perturbed next action by $\tilde{a}\leftarrow \pi_{\nu'}(s') + \epsilon$, $\epsilon$ is sampled from a clipped Gaussian.
	        \STATE Select actions with selector network \\ ~~$\tilde{a}^{(G(\tilde{a}|s'))} \leftarrow G_{\theta'}(\tilde{a}|s') \odot \tilde{a}$
	        \STATE Calculate target critic value $y$ and update critic networks:
	        \STATE ~~$y \leftarrow r + \gamma \min_{i=1,2}Q_{w'_i}(s',\tilde{a}^{(G(\tilde{a}|s'))})$
	        \STATE ~~$w_i \leftarrow \arg\min_{w_i} \mathbf{MSE}(y,Q_{w_i}(s,a^{(G(a|s))}))$ \\
		    \STATE Update $\nu$ by the deterministic policy gradient, with learning rate $\eta_2$:\\
		    ~~$\nu\leftarrow \nu - \eta_2 \nabla_a Q_{w_1}(s,a)|_{a=\pi_\nu(s)}\nabla_\nu \pi_\nu (s)$
		    \STATE Update target networks, with $\tau \in (0,1)$: \\
		    ~~$w'_i \leftarrow \tau {w_i} + (1-\tau)w'_i$ \\
		    ~~$\nu' \leftarrow \tau {\nu} + (1-\tau)\nu'$
		\ENDFOR
\end{algorithmic}
\end{algorithm}

%In this section, we first quantitatively compare IC-INVASE with the vanilla INVASE on synthetic datasets to show its improved scalability. Then we demonstrate IC-INVASE in five continuous control RL tasks with redundant action space where our proposed methods can perform causality-aware RL.

In this section, we demonstrate our proposed methods in five continuous control RL tasks with redundant action space where our proposed methods can perform causality-aware RL. We provide quantitatively comparison between IC-INVASE and the vanilla INVASE on synthetic datasets to show its improved scalability in Appendix~\ref{appd_invase_exp}.

\begin{figure}[t]
\centering
\subfigure[4Rew.-Maze]{
\begin{minipage}[b]{0.15\linewidth}
\label{four_way_maze}
\includegraphics[width=1.0\linewidth]{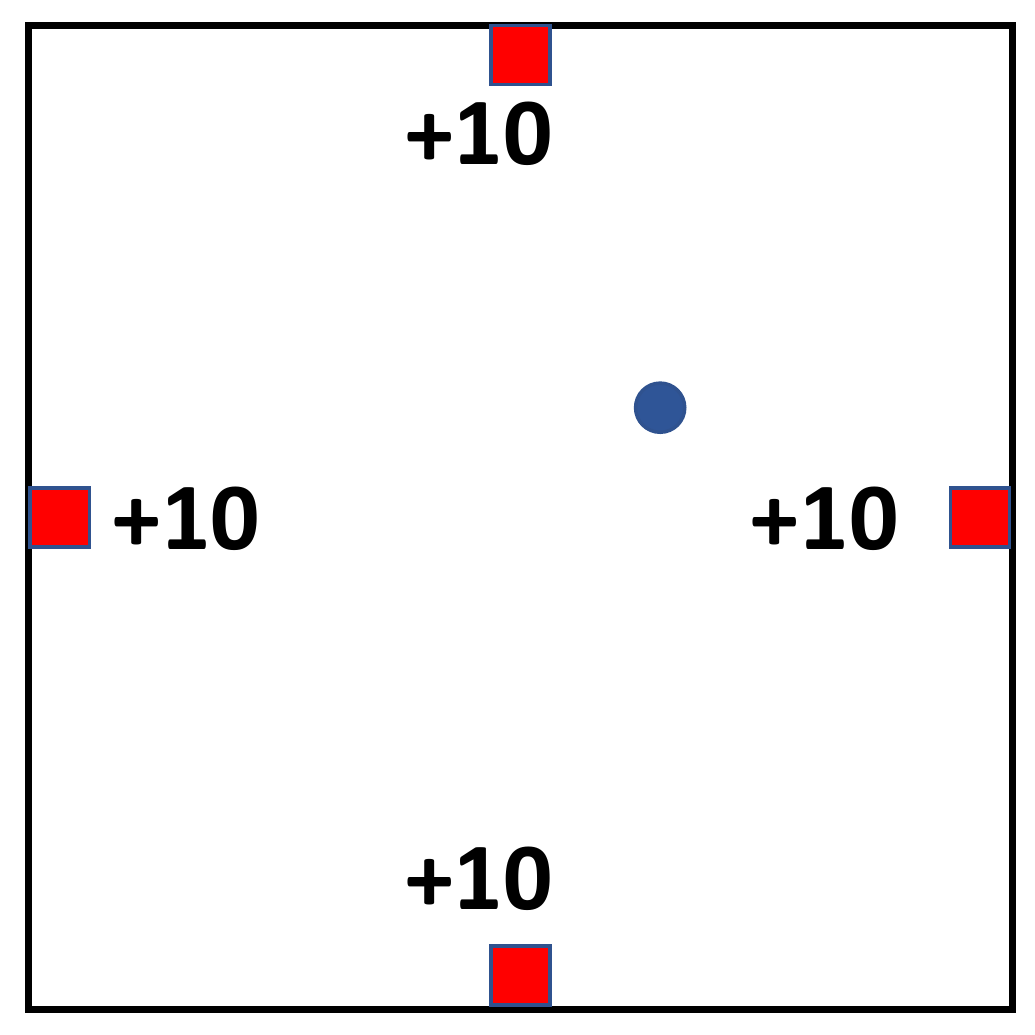}
\end{minipage}}%
\subfigure[Pendulum]{
\begin{minipage}[b]{0.2\linewidth}
\label{pendulum}
\includegraphics[width=1.0\linewidth]{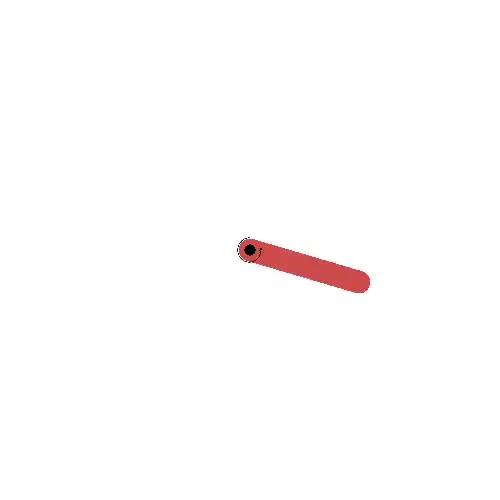}
\end{minipage}}%
\subfigure[Walker2d]{
\begin{minipage}[b]{0.15\linewidth}
\label{pendulum}
\includegraphics[width=1.0\linewidth]{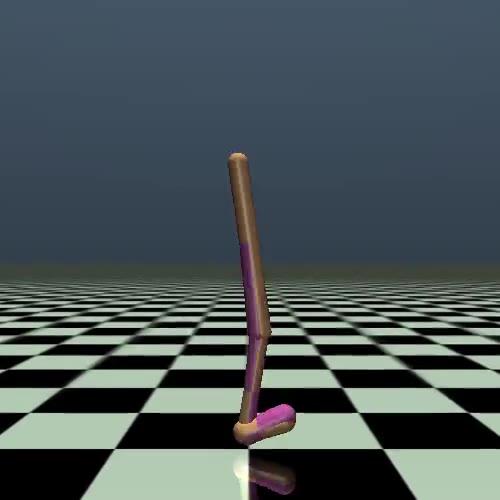}
\end{minipage}}%
\subfigure[LunarLander]{
\begin{minipage}[b]{0.24\linewidth}
\label{pendulum}
\includegraphics[width=1.0\linewidth]{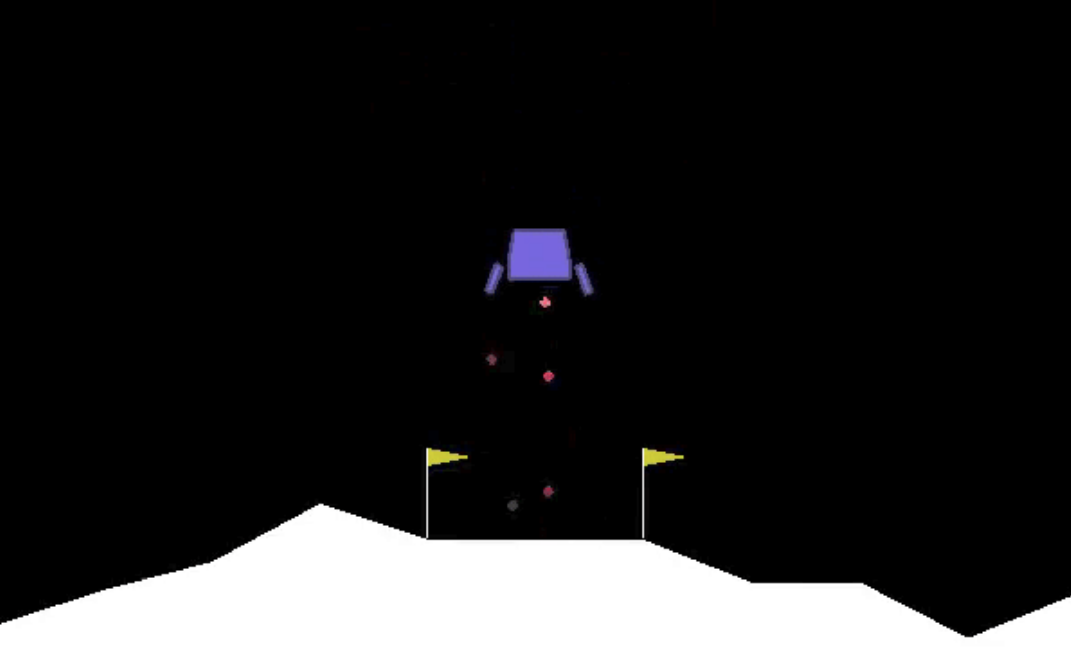}
\end{minipage}}%
\subfigure[BipedalWalker]{
\begin{minipage}[b]{0.22\linewidth}
\label{pendulum}
\includegraphics[width=1.0\linewidth]{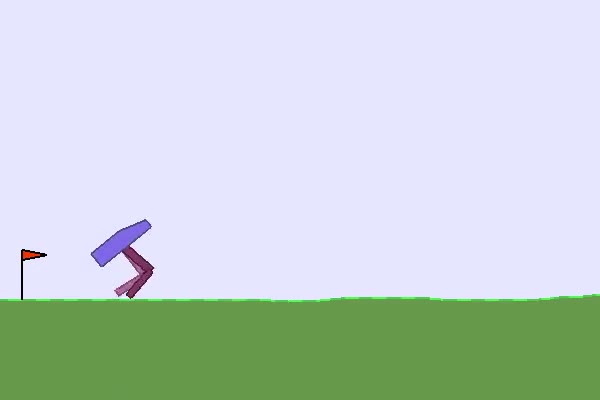}
\end{minipage}}\\
\caption{Environments used in experiments}
\label{fig_envs}
\vspace{-0.3cm}
\end{figure}

\begin{table}[t]
\caption{Tasks used in evaluating SWAR in temporal difference learning}
\label{table_tasks}
\vskip 0.15in
\begin{center}
\begin{small}
\begin{sc}
\begin{tabular}{llcc}
\toprule
Task/Dimension & $|\mathcal{S}|$ & $|\mathcal{A}|$ & $|\mathcal{A}_{red.}|$ \\
\midrule
Pendulum-v0 & 3 & 1 & 100 \\
FourRewardMaze & 2& 2 & 100\\
LunarLanderContinuous-v2 & 8 & 2 & 100\\
BipedalWalker-v3 & 24 & 4 & 100 \\
Walker2d-v2 & 17 & 6 & 100 \\
\bottomrule
\end{tabular}
\end{sc}
\end{small}
\end{center}
\vskip -0.1in
\end{table}

\begin{figure}[t]
\centering
\subfigure[FourRewardMaze]{
\begin{minipage}[b]{0.195\linewidth}
\label{four_way_maze}
\includegraphics[width=1.0\linewidth]{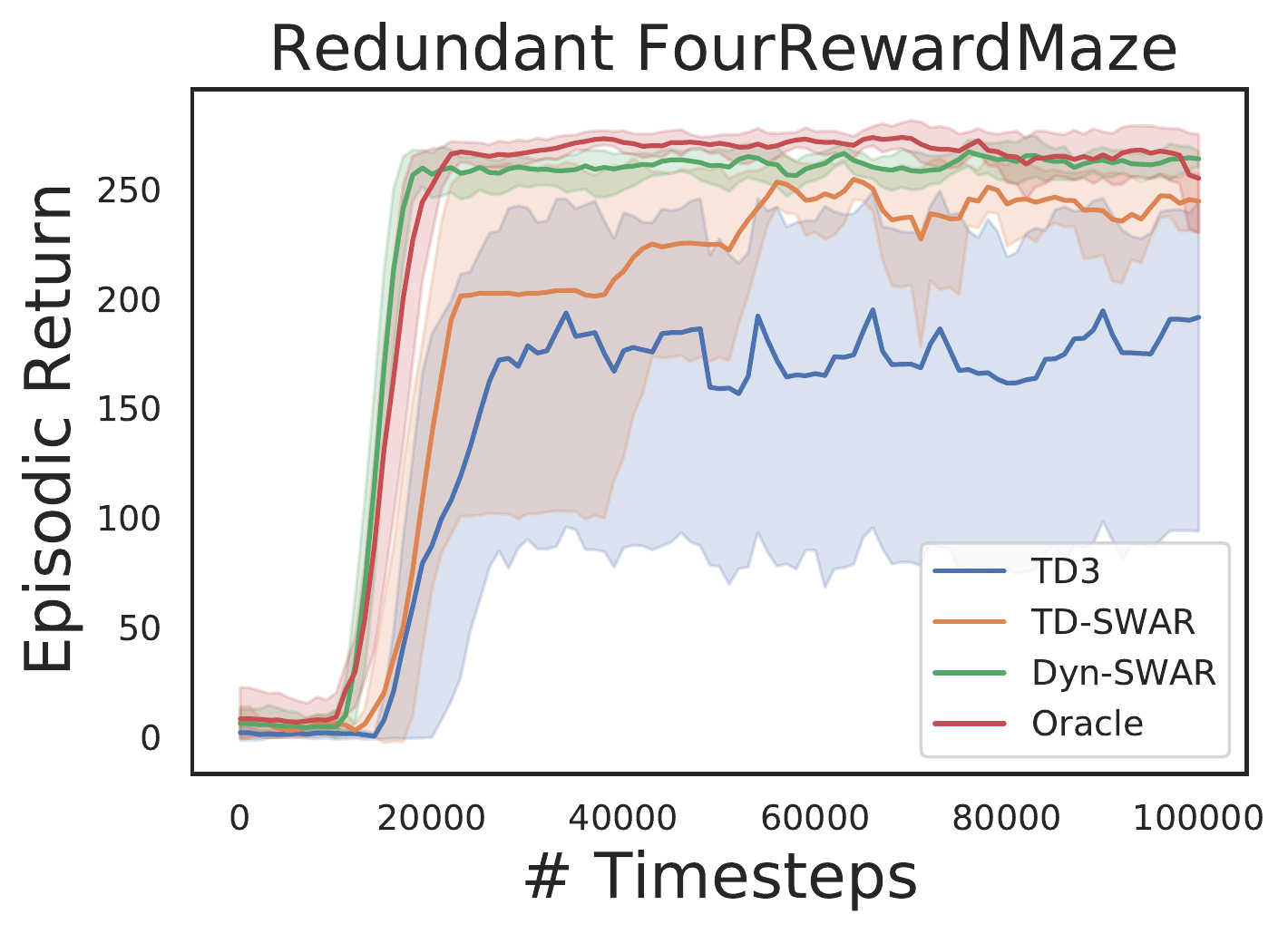}
\end{minipage}}%
\subfigure[Pendulum]{
\begin{minipage}[b]{0.2\linewidth}
\label{pendulum}
\includegraphics[width=1.0\linewidth]{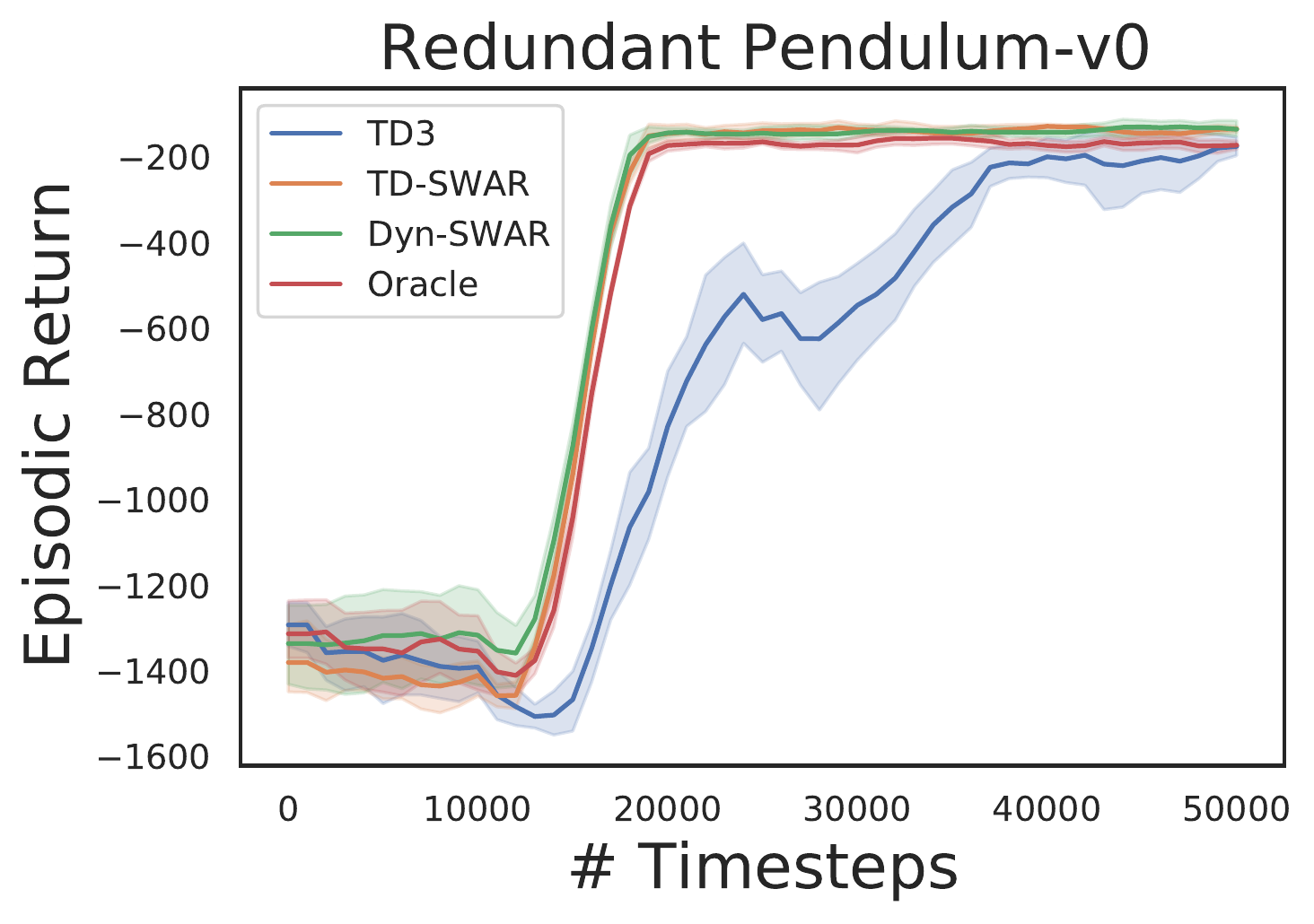}
\end{minipage}}%
\subfigure[Walker2d]{
\begin{minipage}[b]{0.195\linewidth}
\label{pendulum}
\includegraphics[width=1.0\linewidth]{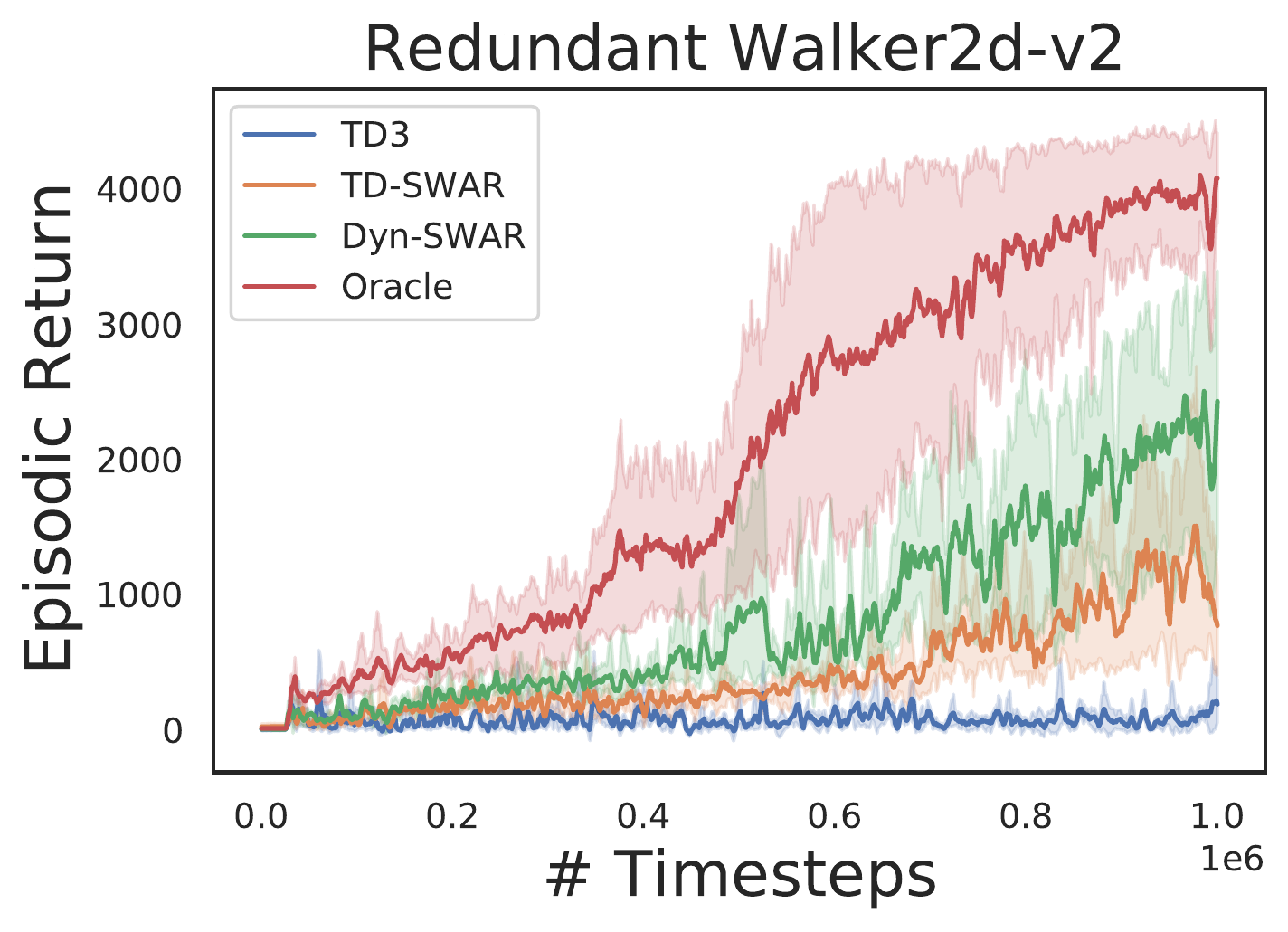}
\end{minipage}}%
\subfigure[LunarLander]{
\begin{minipage}[b]{0.21\linewidth}
\label{pendulum}
\includegraphics[width=1.0\linewidth]{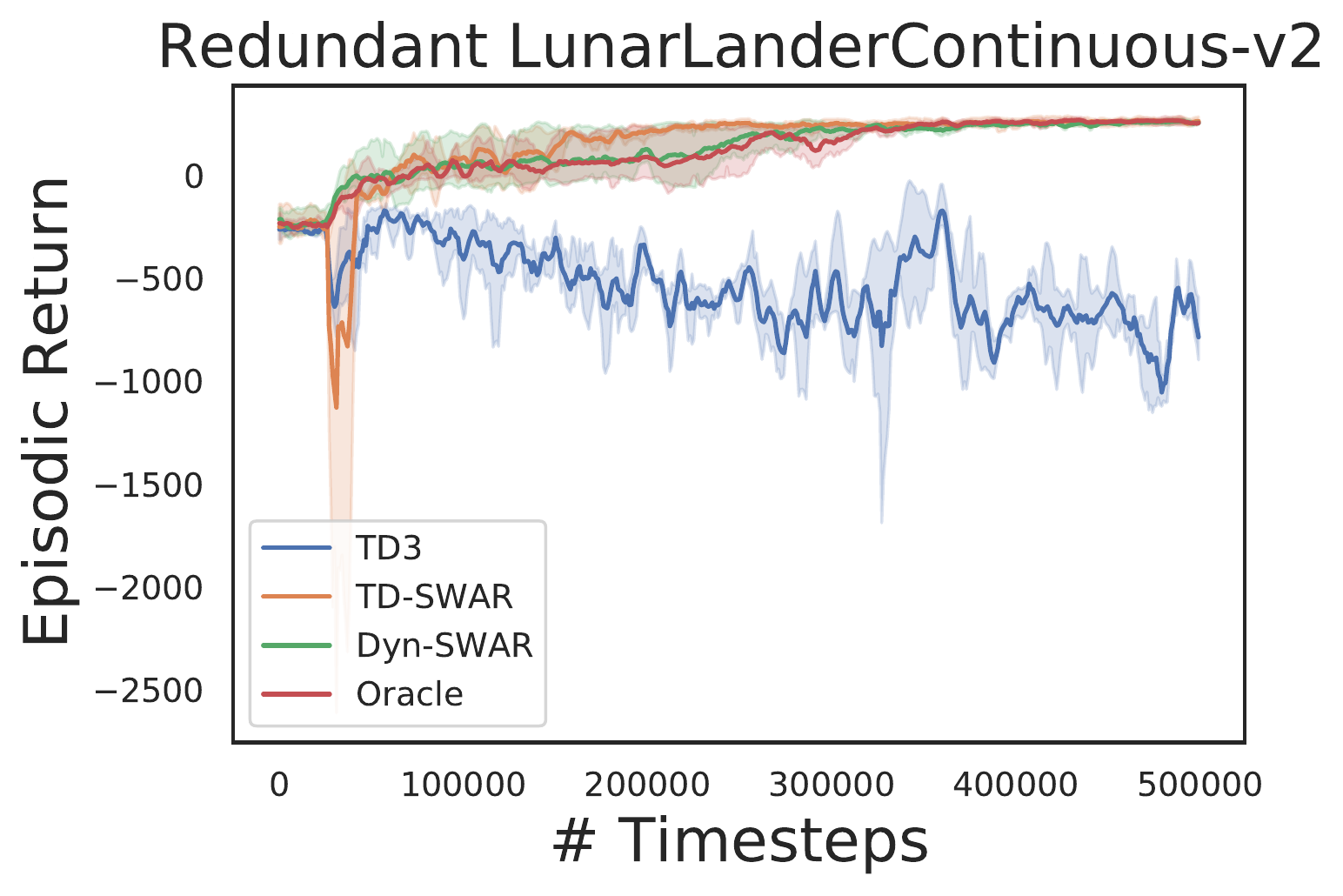}
\end{minipage}}%
\subfigure[BipedalWalker]{
\begin{minipage}[b]{0.195\linewidth}
\label{pendulum}
\includegraphics[width=1.0\linewidth]{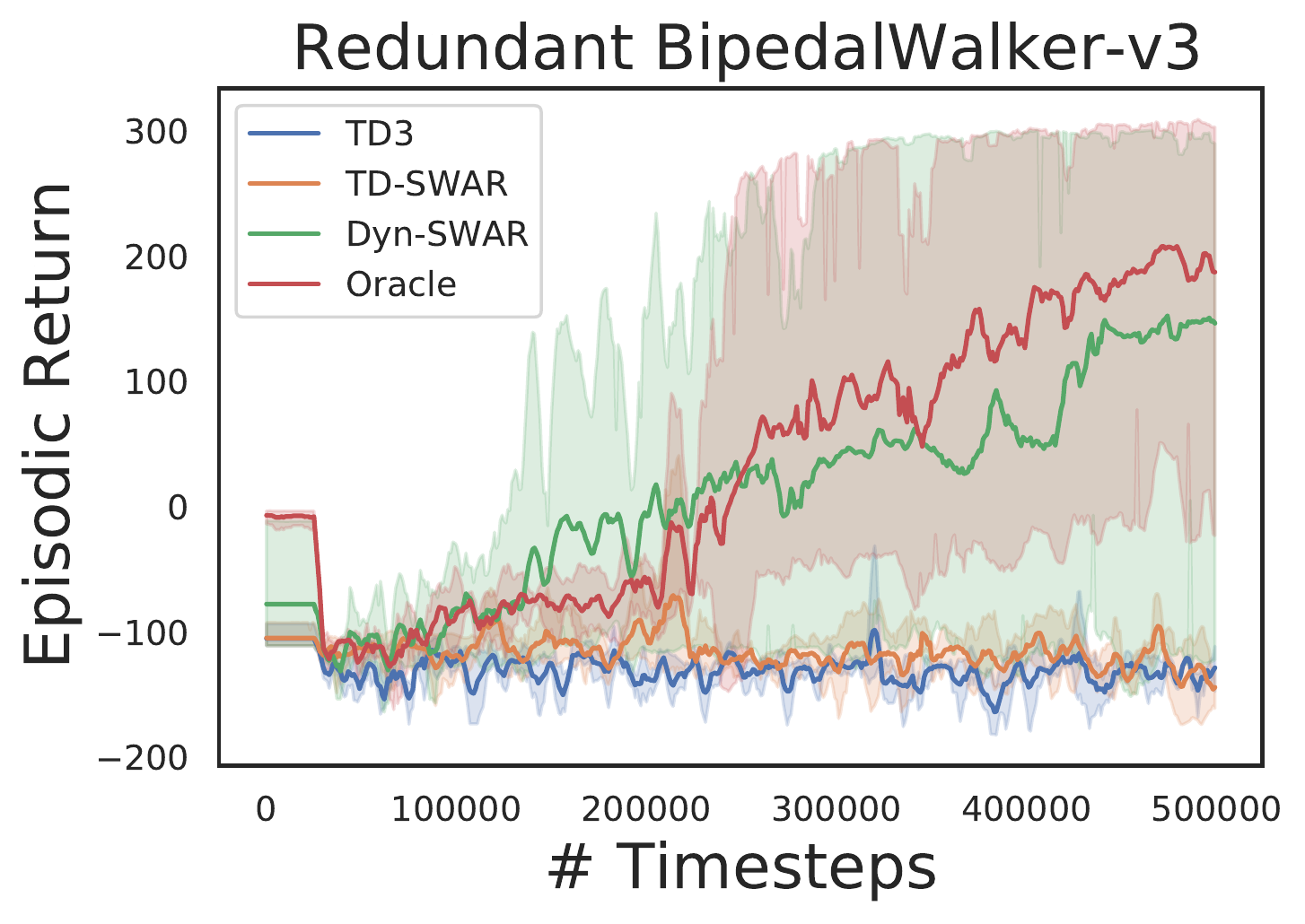}
\end{minipage}}
\caption{Performance of agents in five different environments. The curves shows averaged learning progress and the shaded areas show standard deviation.}
\label{fig_results}
%\vspace{-0.9cm}
\end{figure}
% \subsection{Action Selection in RL}

For this set of experiments, we use five RL environments (Figure~\ref{fig_envs}) that are listed in Table \ref{table_tasks}~\footnote{For more details of the environments please refer to Appendix~\ref{appd_env}}. $|\mathcal{S}|$ means the dimension of state space in each task, $|\mathcal{A}|$ represents the dimension of task-related action space, and $|\mathcal{A}_{red.}|$ indicates the dimension of redundant action space that is injected to each task. Those redundant dimensions of actions will not affect the state transitions or reward calculation, but an agent needs to learn to identify those redundant dimensions to perform efficient learning.

We evaluate both TD-SWAR that integrate IC-INVASE with temporal difference learning and its static variant Dyn-SWAR that applies IC-INVASE in dynamics prediction. The results are compared with two baselines: the \textbf{Oracle}: redundant action dimensions are eliminated manually; and \textbf{TD3}: the vanilla TD3 algorithm without explicit action redundancy reduction. 

In experiments, we find the implementation of Dyn-SWAR can be much more efficient in terms of both sample complexity and computational expense: while the TD-SWAR need to continuously update all parameters for the IC-INVASE selector to keep consistent with the real-time policy and value networks as the regression label varies along time, the Dyn-SWAR selector can be trained with far less amount of data. Say, $10,000$ to $25,000$ timesteps of interactions with the environment. Such a property can be naturally combined with the warm-up trick used in TD3~\cite{fujimoto2018addressing}, i.e., the Dyn-SWAR selector can be trained with warm-up transition tuples collected in the random exploration phase and then be fixed in the later learning process. Compared with normal RL settings where millions of interactions with the environment are always needed, the training of Dyn-SWAR only increases negligible computational expense.

The results are shown in Figure~\ref{fig_results}. In all environments, agent learning with IC-INVASE in both manner (TD-/Dyn-) outperforms the vanilla TD3 baseline. The Dyn-SWAR achieves high learning efficiency that is comparable to the oracle benchmarks. However, the performance of TD-SWAR in higher dimensional tasks (Walker2d-v2 and BipedalWalker-v3) still has a lot of room for improvement. Improving the stability of and scalability of instance-wise variable selection in temporal difference learning thereby should be addressed in future work.

\section{Related Work}
\paragraph{Instance-Wise Feature Selection}
While traditional feature selection method like LASSO~\cite{tibshirani1996regression} aims at finding globally important features across the whole dataset, instance-wise feature selection try to discover the feature-label dependency on a case-by-case basis. L2X~\cite{chen2018learning} performs instance-wise feature selection through mutual information maximization with the technique of Gumbel softmax. L2X requires pre-determined hyper-parameter $k$ to indicate how many features should be selected for each instance, which limits its performance while the number of label-relevant features varies across instances. 

In this work, we build our instance-wise action selection model on top of INVASE~\cite{yoon2018invase}, where policy gradient is applied to replace the Gumbel softmax trick and the size of chosen features per instance is more flexible. 
~\cite{tonekaboni2020went} considers instance-wise feature selection problems in time-series setting, and build generative models to capture counterfactual effects in time series data. Their work enables evaluation of the importance of features over time, which is crucial in the context of healthcare. ~\cite{masoomi2020instance} formally defines different types of feature redundancy and leverages mutual information maximization in instance-wise feature group discovery and introduces theoretical guidance to find the optimal number of different groups. 

Our work is distinguished from previous works for instance-wise feature selection in two aspects. First, while previous works focus on static scenarios like classification and regression, this work focus on temporal difference learning where there is no static label. Second, the scalability of previous methods in variable selection is challenged as there might exist hundreds of redundant actions in the context of RL.

\paragraph{Dimension Reduction in RL}

In the context of RL, attention models~\cite{vaswani2017attention} have been applied to interpret the behaviors of learned policies. ~\cite{tang2020neuroevolution} proposes to perceive the state information through a self-attention bottleneck in vision-based RL tasks, which concentrates on the state space redundancy reduction with image inputs. %While their proposed method can provide some interpretable insights of the learned policy, the generalization ability is limited and the attention module can not generalize to unseen noisy inputs. 
The work of ~\cite{mott2019towards} also applies the attention mechanism to learn task-relevant information. The proposed method achieves state-of-the-art performance on Atari games with image input while being more understandable with top-down attention models.

Different from those papers, this work considers relatively tight state representations (vector input), and focuses on the task-irrelevant action reduction. We aim at finding the task-related actions and improving the learning efficiency without wasting samples in learning the task-irrelevant dimensions of actions. Our work is most closely related to AE-DQN ~\cite{zahavy2018learn} in that we both consider the problem of redundant action elimination. AE-DQN tackles action space redundancy with an action-elimination network that eliminates sub-optimal actions. Yet its discussion is limited in the discrete settings. In contrast, our work focuses on action elimination in continuous control tasks.

%\paragraph{Curriculum Learning}

\section{Conclusion and Future Work}
In this work, we tackle the challenge of action space pruning in action redundant RL tasks. Recent advance on instance-wise feature selection technique (INVASE) is exploited after curriculum learning and iterative operation are integrated for the pursuance of scalability and efficiency. The resulting method, termed IC-INVASE, is then generalized to the RL setting where two different algorithms are proposed, TD-SWAR and Dyn-SWAR, to conduct causality-aware RL. While the former algorithm addresses the action redundant issue directly in temporal difference learning, the latter algorithm captures dynamical causality with model-based prediction. Experiments on various tasks demonstrate the causality-awareness is crucial for RL agents to perform efficient learning in action-redundant environments.

In future work, the iterative property can be further explored to perform ensemble methods in variable selection. And a more proper curriculum might be designed to better fuse multiple curricula together. On the RL side, the stability of TD-SWAR might be further improved for better sample efficiency. % Curriculum design can potentially be benefit, e.g., an agent may first learn to identify actions that are important in general before focusing on discovering state-dependent important actions. Moreover, the selection can be generalized to both state space and action space to conduct efficient temporal difference learning with the causality-awareness among states, actions and the task. The model-based prediction can be extended to future return prediction.


\bibliography{ref}

\begin{thebibliography}{}

\bibitem[\protect\citeauthoryear{Abadi \bgroup \em et al.\egroup
  }{2016}]{abadi2016tensorflow}
Mart{\'\i}n Abadi, Paul Barham, Jianmin Chen, et~al.
\newblock Tensorflow: A system for large-scale machine learning.
\newblock In {\em OSDI}, 2016.

\bibitem[\protect\citeauthoryear{Andrychowicz \bgroup \em et al.\egroup
  }{2020}]{andrychowicz2020learning}
OpenAI:~Marcin Andrychowicz, Bowen Baker, Maciek Chociej, et~al.
\newblock Learning dexterous in-hand manipulation.
\newblock {\em IJRR}, 2020.

\bibitem[\protect\citeauthoryear{Bengio \bgroup \em et al.\egroup
  }{2009}]{bengio2009curriculum}
Yoshua Bengio, J{\'e}r{\^o}me Louradour, et~al.
\newblock Curriculum learning.
\newblock In {\em ICML}, 2009.

\bibitem[\protect\citeauthoryear{Berner \bgroup \em et al.\egroup
  }{2019}]{berner2019dota}
Christopher Berner, Greg Brockman, Brooke Chan, et~al.
\newblock Dota 2 with large scale deep reinforcement learning.
\newblock {\em arXiv:1912.06680}, 2019.

\bibitem[\protect\citeauthoryear{Chen \bgroup \em et al.\egroup
  }{2018}]{chen2018learning}
Jianbo Chen, Le~Song, et~al.
\newblock Learning to explain: An information-theoretic perspective on model
  interpretation.
\newblock {\em arXiv:1802.07814}, 2018.

\bibitem[\protect\citeauthoryear{Chollet and others}{2015}]{chollet2015keras}
Fran\c{c}ois Chollet et~al.
\newblock Keras.
\newblock \url{https://github.com/fchollet/keras}, 2015.

\bibitem[\protect\citeauthoryear{Czarnecki \bgroup \em et al.\egroup
  }{2018}]{czarnecki2018mix}
Wojciech~Marian Czarnecki, Siddhant~M Jayakumar, Max Jaderberg, et~al.
\newblock Mix\&match-agent curricula for reinforcement learning.
\newblock {\em arXiv:1806.01780}, 2018.

\bibitem[\protect\citeauthoryear{de Haan \bgroup \em et al.\egroup
  }{2019}]{de2019causal}
Pim de~Haan, Dinesh Jayaraman, and Sergey Levine.
\newblock Causal confusion in imitation learning.
\newblock In {\em NeurIPS}, pages 11698--11709, 2019.

\bibitem[\protect\citeauthoryear{Even-Dar \bgroup \em et al.\egroup
  }{2006}]{even2006action}
Eyal Even-Dar, Shie M., et~al.
\newblock Action elimination and stopping conditions for the multi-armed bandit
  and reinforcement learning problems.
\newblock {\em JMLR}, 2006.

\bibitem[\protect\citeauthoryear{Fujimoto \bgroup \em et al.\egroup
  }{2018}]{fujimoto2018addressing}
Scott Fujimoto, Herke Van~Hoof, and David Meger.
\newblock Addressing function approximation error in actor-critic methods.
\newblock {\em arXiv:1802.09477}, 2018.

\bibitem[\protect\citeauthoryear{Ha and Schmidhuber}{2018}]{ha2018world}
David Ha and J{\"u}rgen Schmidhuber.
\newblock World models.
\newblock {\em arXiv:1803.10122}, 2018.

\bibitem[\protect\citeauthoryear{Haarnoja \bgroup \em et al.\egroup
  }{2018}]{haarnoja2018soft}
Tuomas Haarnoja, Aurick Zhou, et~al.
\newblock Soft actor-critic: Off-policy maximum entropy deep reinforcement
  learning with a stochastic actor.
\newblock {\em arXiv:1801.01290}, 2018.

\bibitem[\protect\citeauthoryear{Hafner \bgroup \em et al.\egroup
  }{2019}]{hafner2019dream}
Danijar Hafner, Timothy Lillicrap, Jimmy Ba, et~al.
\newblock Dream to control: Learning behaviors by latent imagination.
\newblock {\em arXiv:1912.01603}, 2019.

\bibitem[\protect\citeauthoryear{Janner \bgroup \em et al.\egroup
  }{2019}]{janner2019trust}
Michael Janner, Justin Fu, Marvin Zhang, et~al.
\newblock When to trust your model: Model-based policy optimization.
\newblock In {\em NeurIPS}, 2019.

\bibitem[\protect\citeauthoryear{Konda and Tsitsiklis}{2000}]{konda2000actor}
Vijay~R Konda and John~N Tsitsiklis.
\newblock Actor-critic algorithms.
\newblock In {\em NeurIPS}, 2000.

\bibitem[\protect\citeauthoryear{Langlois \bgroup \em et al.\egroup
  }{2019}]{langlois2019benchmarking}
Eric Langlois, Shunshi Zhang, Guodong Zhang, et~al.
\newblock Benchmarking model-based reinforcement learning.
\newblock {\em arXiv:1907.02057}, 2019.

\bibitem[\protect\citeauthoryear{Lillicrap \bgroup \em et al.\egroup
  }{2015}]{lillicrap2015continuous}
Timothy~P Lillicrap, Jonathan~J Hunt, Alexander Pritzel, et~al.
\newblock Continuous control with deep reinforcement learning.
\newblock {\em arXiv:1509.02971}, 2015.

\bibitem[\protect\citeauthoryear{Masoomi \bgroup \em et al.\egroup
  }{2020}]{masoomi2020instance}
Aria Masoomi, Chieh Wu, Tingting Zhao, et~al.
\newblock Instance-wise feature grouping.
\newblock {\em NeurIPS}, 2020.

\bibitem[\protect\citeauthoryear{Matiisen \bgroup \em et al.\egroup
  }{2019}]{matiisen2019teacher}
Tambet Matiisen, Avital Oliver, et~al.
\newblock Teacher-student curriculum learning.
\newblock {\em TNNLS}, 2019.

\bibitem[\protect\citeauthoryear{Melnik \bgroup \em et al.\egroup
  }{}]{melniktactile}
Andrew Melnik, Luca Lach, Matthias Plappert, et~al.
\newblock Tactile sensing and deep reinforcement learning for in-hand
  manipulation tasks.

\bibitem[\protect\citeauthoryear{Mnih \bgroup \em et al.\egroup
  }{2015}]{mnih2015human}
Volodymyr Mnih, Koray Kavukcuoglu, David Silver, et~al.
\newblock Human-level control through deep reinforcement learning.
\newblock {\em Nature}, 2015.

\bibitem[\protect\citeauthoryear{Mott \bgroup \em et al.\egroup
  }{2019}]{mott2019towards}
Alexander Mott, Daniel Zoran, et~al.
\newblock Towards interpretable reinforcement learning using attention
  augmented agents.
\newblock In {\em NeurIPS}, 2019.

\bibitem[\protect\citeauthoryear{Paszke \bgroup \em et al.\egroup
  }{2017}]{paszke2017automatic}
Adam Paszke, Sam Gross, Soumith Chintala, et~al.
\newblock Automatic differentiation in pytorch.
\newblock 2017.

\bibitem[\protect\citeauthoryear{Sharma \bgroup \em et al.\egroup
  }{2019}]{sharma2019dynamics}
Archit Sharma, Shixiang Gu, Sergey Levine, et~al.
\newblock Dynamics-aware unsupervised discovery of skills.
\newblock {\em arXiv:1907.01657}, 2019.

\bibitem[\protect\citeauthoryear{Silver \bgroup \em et al.\egroup
  }{2014}]{silver2014deterministic}
David Silver, Guy Lever, et~al.
\newblock Deterministic policy gradient algorithms.
\newblock In {\em ICML}, 2014.

\bibitem[\protect\citeauthoryear{Silver \bgroup \em et al.\egroup
  }{2016}]{silver2016mastering}
David Silver, Aja Huang, Chris~J Maddison, et~al.
\newblock Mastering the game of go with deep neural networks and tree search.
\newblock {\em nature}, 2016.

\bibitem[\protect\citeauthoryear{Sun \bgroup \em et al.\egroup
  }{2018}]{sun2018tstarbots}
Peng Sun, Xinghai Sun, Lei Han, et~al.
\newblock Tstarbots: Defeating the cheating level builtin ai in starcraft ii in
  the full game.
\newblock {\em arXiv:1809.07193}, 2018.

\bibitem[\protect\citeauthoryear{Sutton and
  Barto}{1998}]{sutton1998reinforcement}
Richard~S Sutton and Andrew~G Barto.
\newblock Reinforcement learning: An introduction.
\newblock 1998.

\bibitem[\protect\citeauthoryear{Tang \bgroup \em et al.\egroup
  }{2020}]{tang2020neuroevolution}
Yujin Tang, Duong Nguyen, and David Ha.
\newblock Neuroevolution of self-interpretable agents.
\newblock {\em arXiv:2003.08165}, 2020.

\bibitem[\protect\citeauthoryear{Tibshirani}{1996}]{tibshirani1996regression}
Robert Tibshirani.
\newblock Regression shrinkage and selection via the lasso.
\newblock {\em JRSS}, 1996.

\bibitem[\protect\citeauthoryear{Tonekaboni \bgroup \em et al.\egroup
  }{2020}]{tonekaboni2020went}
Sana Tonekaboni, S.~Joshi, et~al.
\newblock What went wrong and when? instance-wise feature importance for
  time-series black-box models.
\newblock {\em NeurIPS}, 2020.

\bibitem[\protect\citeauthoryear{Vaswani \bgroup \em et al.\egroup
  }{2017}]{vaswani2017attention}
Ashish Vaswani, Noam Shazeer, et~al.
\newblock Attention is all you need.
\newblock In {\em NeurIPS}, 2017.

\bibitem[\protect\citeauthoryear{Vinyals \bgroup \em et al.\egroup
  }{2019}]{vinyals2019grandmaster}
Oriol Vinyals, Igor Babuschkin, Wojciech~M Czarnecki, et~al.
\newblock Grandmaster level in starcraft ii using multi-agent reinforcement
  learning.
\newblock {\em Nature}, 2019.

\bibitem[\protect\citeauthoryear{Weinshall \bgroup \em et al.\egroup
  }{2018}]{weinshall2018curriculum}
Daphna Weinshall, Gad Cohen, et~al.
\newblock Curriculum learning by transfer learning: Theory and experiments with
  deep networks.
\newblock {\em arXiv:1802.03796}, 2018.

\bibitem[\protect\citeauthoryear{Williams}{1992}]{williams1992simple}
Ronald~J Williams.
\newblock Simple statistical gradient-following algorithms for connectionist
  reinforcement learning.
\newblock {\em Machine learning}, 8(3-4):229--256, 1992.

\bibitem[\protect\citeauthoryear{Xu \bgroup \em et al.\egroup
  }{2020}]{xu2020curriculum}
Benfeng Xu, L.~Zhang, et~al.
\newblock Curriculum learning for natural language understanding.
\newblock In {\em ACL}, 2020.

\bibitem[\protect\citeauthoryear{Yoon \bgroup \em et al.\egroup
  }{2018}]{yoon2018invase}
Jinsung Yoon, James Jordon, and Mihaela van~der Schaar.
\newblock Invase: Instance-wise variable selection using neural networks.
\newblock In {\em ICLR}, 2018.

\bibitem[\protect\citeauthoryear{Zahavy \bgroup \em et al.\egroup
  }{2018}]{zahavy2018learn}
Tom Zahavy, Matan Haroush, et~al.
\newblock Learn what not to learn: Action elimination with deep reinforcement
  learning.
\newblock In {\em NeurIPS}, 2018.

\end{thebibliography}
\bibliographystyle{named}
\onecolumn
\newpage
\appendix

\section{On the Dynamic Model Approximation}
\label{appd_apprx}
We provide analysis on the approximation in this section based on the deterministic MDP model in finite action space where the problem degenerates to $Q$-Learning. Similar results can be get to prove the Policy Evaluation Lemma, combined with Policy Improvement Lemma (given proper function approximation of the $\arg\max$ operator) and result in Policy Iteration Theorem.

In deterministic MDPs with $s_{t+1} = \mathcal{T}(s_t,a_t)$, $r_t = r(s_t,a_t)$, the value function of a state is defined as 
\begin{equation}
    V^{\pi}(s) = \sum_{t=0}^{\infty} \gamma^t r(s_t,a_t),
\end{equation}
given $s_0 = s$ is the initial state and $a_t=\pi(s_t)$ comes from the deterministic policy $\pi$.

The learning objective is to find an optimal policy $\pi$, such that an optimal state value can be achieved:
\begin{equation}
    V^*(s) = \max_\pi V^{\pi}(s)
\end{equation}
The state-action value function ($Q$-function) is then defined as \begin{equation}
    Q(s,a) = r(s,a) + \gamma V^*(\mathcal{T}(s,a))
\end{equation}
Formally, the objective of action space pruning in action-redundant MDPs is to find an optimal policy $\pi^{(G)} = G(\pi(s_t)|s_t)\odot \pi(s_t)$ with an action selector $G:\mathcal{S}\times\mathcal{A}\mapsto \{0,1\}^d$, 
\begin{equation}
\label{eq_appd1}
    V^*(s) = \max_{\pi^{(G)}}V^{\pi^{(G)}}(s) =\max_\pi V^{\pi}(s),
\end{equation}
with minimal number of actions selected, i.e., $|G|_0$ is minimized. The sufficient and necessary condition for Equation~(\ref{eq_appd1}) to hold is $r(s_t,\pi(s_t)) = r(s_t, \pi^{(G)}{(s_t)})$ and $\mathcal{T}(s_t,\pi(s_t)) = \mathcal{T}(s_t, \pi^{(G)}(s_t))$.

In general, the reward function $r$ and transition dynamics $\mathcal{T}$ may depend on different subsets of actions and the optimal, i.e., $r(s_t,a_t) = r(s_t, a_t^{(G_1)})$, while $\mathcal{T}(s_t,a_t) = \mathcal{T}(s_t, a_t^{(G_2)})$, where $G_1$, $G_2$ select different subset of given actions by $a_t^{(G_1)} = G_1(a_t|s_t) \odot a_t$, $a_t^{(G_2)} = G_2(a_t|s_t) \odot a_t$ but $a_t^{(G_1)}\ne a_t^{(G_2)}$. The final action selector $G$ should be generated according to $G(a|s) = G_1(a|s) \lor G_2(a|s)$, where $\lor$ is the element-wise $\mathbf{OR}$ operation.

Therefore, in our approximation of Dyn-SWAR, we assume $G(a|s)= G_2(a|s)$ as an approximation for $G(a|s) = G_1(a|s) \lor G_2(a|s) $. Future work may include another predictive model for the reward function and take the element-wise $\mathbf{OR}$ operation to get $G$.

\section{Additional Experiments}
\label{appd_invase_exp}
\subsection{Synthetic Data Experiment}
The synthetic datasets are generated in the same way as~\cite{chen2018learning,yoon2018invase}. Specifically, there are $6$ synthetic datasets that have inputs generated from an $11$-dim Gaussian distribution without correlations across features. The label $Y$ for each dataset is generated by a Bernoulli random variable with $P(Y=1|X) = \frac{1}{1+\text{logit}(X)}$. In different tasks, $\text{logit}(X)$ takes the value of:
\begin{itemize}
    \item $\mathbf{Syn1}$: $\exp(X_1 X_2)$
    \item $\mathbf{Syn2}$: $\exp(\sum_{i=3}^{6} X_i^2 -4)$
    \item $\mathbf{Syn3}$: $-10 \times \sin{2X_7} + 2|X_8| + X_9 + \exp(-X_{10})$
    \item $\mathbf{Syn4}$: if $X_{11}<0$, logit follows $\mathbf{Syn1}$, otherwise, logit follows $\mathbf{Syn2}$
    \item $\mathbf{Syn5}$: if $X_{11}<0$, logit follows $\mathbf{Syn1}$, otherwise, logit follows $\mathbf{Syn3}$
    \item $\mathbf{Syn6}$: if $X_{11}<0$, logit follows $\mathbf{Syn2}$, otherwise, logit follows $\mathbf{Syn3}$
\end{itemize}
In the first three synthetic datasets, the label $Y$ depends on the same feature across each dataset, while in the last three datasets, the subsets of features that label $Y$ depends on are determined by the values of $X_{11}$.

For each dataset, $20,000$ samples are generated and be separated into a training set and a testing set. In this work, we focus on finding outcome-relevant features (e.g., finding task-relevant actions in the context of RL), thus the true positive rate (TPR) and false discovery rate (FDR) are used as performance metrics.

\paragraph{11-dim Feature Selection}
Table \ref{table_11features} shows the quantitative results of the proposed method, IC-INVASE on the 11-dim feature selection tasks. To accelerate training and facilitate the usage of dynamical computational graphs in curriculum learning and RL settings, the vanilla INVASE is re-implemented with PyTorch~\cite{paszke2017automatic}. In general, the PyTorch implementation is $4$ to $5$ times faster than the previous Keras~\cite{abadi2016tensorflow,chollet2015keras} implementation, with on-par performance on the $11$-dim feature selection tasks. In the comparison, both the reported results in ~\cite{yoon2018invase} (denoted by \textbf{INVASE (REP.)}) and our experimental results on INVASE (denoted by \textbf{INVASE (EXP.)}) are presented. The $p_r$ curriculum for IC-INVASE in all experiments are set to decrease from $0.5$ to $0.0$ except in ablation studies. Results of two different choices of the $\lambda$ curriculum are reported and denoted by \textbf{IC-INVASE ($\mathbf{\lambda \uparrow \cdot}$)}, e.g., $\lambda \uparrow 0.3$ means $\lambda$ increases from $0.0$ to $0.3$ in the experiment.
We omit the results on the first three datasets ($\mathbf{Syn1}$,$\mathbf{Syn2}$,$\mathbf{Syn3}$) where both IC-INVASE and INVASE achieve $100.0$ TPR and $0.0$ FDR. Iteration $1$ to Iteration $4$ in the table shows the results after applying the selection operator for different number of iterations.

In all experiments, IC-INVASE achieves better performance (i.e., larger TPR and lower FDR) than the vanilla INVASE with Keras and PyTorch implementation. Iterative applying the feature selection operator can reduce the FDR with a slight cost of TPR decay.

\paragraph{100-dim Feature Selection}
We then increase the total number of feature dimensions to $100$ to demonstrate how IC-INVASE improves the vanilla INVASE in larege-scale variable selection settings. In this experiment. The features are generated with $100$-dim Gaussian without correlations and the rules for label generation are still the same as the $11$-dim settings. (i.e., $89$ additional label-independent noisy dimensions of input is concatenated to the $11$-dim inputs.)

The results are shown in Table \ref{table_100features}. IC-INVASE achieves much better performance in all datasets, i.e., higher TPR and lower FDR. 
The ablation studies on different curriculum show both an increasing $\lambda$ and a decreasing $p_r$ can benefit discovery of label-dependent features. As the hyper-parameters for curriculum are not elaborated in our experiments, direct combining the two curriculum may hinder the performance. The design for curriculum fusion is left to the future work.

\begin{table*}[t]
\caption{Relevant variables discovery results for Synthetic datasets with 11-dim input}
\label{table_11features}
\vskip 0.15in
\small
\begin{center}
\scalebox{0.925}{
\begin{sc}
\begin{tabular}{llcccccccc}
\toprule
Data Set & Method &\multicolumn{2}{r}{Iteration $1$} & \multicolumn{2}{r}{Iteration $2$} & \multicolumn{2}{r}{Iteration $3$} & \multicolumn{2}{r}{Iteration $4$}\\
%\cline{1}% \cline{3-4} \cline{5-6} \cline{7-8} \cline{9-10}
\cmidrule(r){1-1} \cmidrule(r){3-4}  \cmidrule(r){5-6} \cmidrule(r){7-8} \cmidrule(r){9-10}
Metric & & TPR & FDR &TPR & FDR &TPR & FDR &TPR & FDR \\
\midrule
\multirow{4}{*}{$\mathbf{Syn 4}$}    & INVASE (rep.) & 99.8 & 10.3 & & & & & & \\
                                    & INVASE (exp.)& 98.6 & 1.6 & 98.1& 1.1&98.1& 1.1 &98.1& 1.1 \\
                                    & IC-INVASE ($\lambda \uparrow 0.2$)& 99.7 & 3.4 & 99.7& 2.6& 99.7&2.5 &99.7 &2.5 \\
                                    & IC-INVASE ($\lambda \uparrow 0.3$)& 99.3 & 1.6 &\textbf{99.3} &\textbf{0.8} &\textbf{99.3} &\textbf{0.8} &\textbf{99.3} & \textbf{0.8}\\
                                    \hline
\multirow{4}{*}{$\mathbf{Syn 5}$}    & INVASE (rep.) & 84.8 & 1.1 & & & & & & \\
                                    & INVASE (exp.)& 82.1 & 1.0 & 79.7& 1.0& 79.3& 1.0 &79.2& 1.0 \\
                                    & IC-INVASE ($\lambda \uparrow 0.2$)& 99.3 & 1.6 & 99.1 & 1.1 & 99.1 &1.1 &99.1 &1.1 \\
                                    & IC-INVASE ($\lambda \uparrow 0.3$)& 96.8 & 1.0 & \textbf{96.4} & \textbf{0.4} &\textbf{96.4} &\textbf{0.4} &\textbf{96.4} &\textbf{0.4} \\
                                    \hline
\multirow{4}{*}{$\mathbf{Syn 6}$}    & INVASE (rep.) & 90.1 & 7.4 & & & & & & \\
                                    & INVASE (exp.)& 92.3 & 1.7 & 89.8& 1.6&89.6&1.6 &89.6& 1.6 \\
                                    & IC-INVASE ($\lambda \uparrow 0.2$)& 99.6 & 2.9 & 99.5 & 2.6& 99.5&2.5 & 99.5& 2.5\\
                                    & IC-INVASE ($\lambda \uparrow 0.3$)& 99.4 & 1.9 & \textbf{99.3} &\textbf{1.6} &\textbf{99.3} &\textbf{1.6} &\textbf{99.3} &\textbf{1.6} \\
\bottomrule
\end{tabular}
\end{sc}}
\end{center}
\vskip -0.1in
\end{table*}

\begin{table*}[t]
\caption{Relevant feature discovery results for Synthetic datasets with 100-dim input}
\label{table_100features}
\vskip 0.15in
\begin{center}
\scalebox{0.815}{
\begin{sc}
\begin{tabular}{llcccccccc}
\toprule
Data Set & Method &\multicolumn{2}{r}{Iteration $1$} & \multicolumn{2}{r}{Iteration $2$} & \multicolumn{2}{r}{Iteration $3$} & \multicolumn{2}{r}{Iteration $4$}\\
%\cline{1}% \cline{3-4} \cline{5-6} \cline{7-8} \cline{9-10}
\cmidrule(r){1-1} \cmidrule(r){3-4}  \cmidrule(r){5-6} \cmidrule(r){7-8} \cmidrule(r){9-10}
Metric & & TPR & FDR &TPR & FDR &TPR & FDR &TPR & FDR \\
\midrule
\multirow{5}{*}{$\mathbf{Syn 4}$}    & INVASE (rep.) & 66.3 & 40.5 & & & & & & \\
                                    & INVASE (exp.)& 27.0 & 6.5 & 18.0& 6.4&18.0& 6.4 &18.0& 6.4 \\
                                    & IC-INVASE W/O $p_r\downarrow$ & 66.3 & 40.5 &66.3 &40.5 &66.3  &40.5 &66.3 & 40.5\\
                                    & IC-INVASE W/O $\lambda \uparrow$ & 100.0 & 43.0 & 100.0& 43.0& 100.0&43.0 &100.0 &43.0 \\
                                     & IC-INVASE&\textbf{ 100.0 }& 43.0 & \textbf{100.0}& 43.0& \textbf{100.0}&43.0 &\textbf{100.0} &43.0 \\
                                    \hline
\multirow{5}{*}{$\mathbf{Syn 5}$}    & INVASE (rep.) & 73.2 & 23.7 & & & & & & \\
                                    & INVASE (exp.)& 56.4 & 37.9 & 56.4& 37.9&56.4& 37.9 &56.4& 37.9 \\
                                    & IC-INVASE W/O $p_r\downarrow$ & 90.9 & 7.8 & 88.8 &4.4 &88.8  &4.3 &88.8 & 4.3\\
                                    & IC-INVASE W/O $\lambda \uparrow$ & 96.1 & 11.3 & 95.2& 8.2& \textbf{95.5}&8.1 &\textbf{95.5} &8.1 \\
                                    & IC-INVASE& 91.9 & 8.1 & 90.8& 4.3& \textbf{90.8}&\textbf{4.2} &\textbf{90.8} &\textbf{4.2} \\
                                    \hline
\multirow{5}{*}{$\mathbf{Syn 6}$}    & INVASE (rep.) & 90.5 & 15.4 & & & & & & \\
                                    & INVASE (exp.)& 90.1 & 43.7 & 90.1 & 43.7&90.1 & 43.7 &90.1 & 43.7\\
                                    & IC-INVASE W/O $p_r\downarrow$ & 98.5 & 4.1 &98.4 &2.4 &98.4  &\textbf{2.3} &98.4 & \textbf{2.3}\\
                                    & IC-INVASE W/O $\lambda \uparrow$ & 99.6 & 8.1 & 99.6& 7.1& \textbf{99.6}&7.0 &\textbf{99.6} &7.0 \\
                                    & IC-INVASE& 98.9 & 7.0 & 98.9& 5.0& \textbf{98.9}&\textbf{4.9} &\textbf{98.9} &\textbf{4.9} \\
\bottomrule
\end{tabular}
\end{sc}}
\end{center}
\vskip -0.1in
\end{table*}

\section{Environment Details}
\label{appd_env}

\begin{figure}[h]
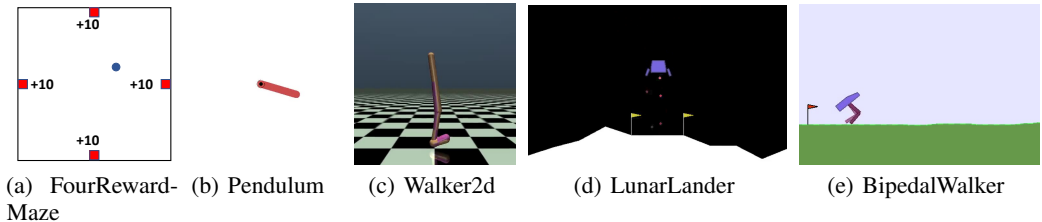

\centering
\subfigure[FourRewardMaze]{
\begin{minipage}[b]{0.153\linewidth}
\label{four_way_maze}
\includegraphics[width=1.0\linewidth]{fig/FourWayMaze.pdf}
\end{minipage}}%
\subfigure[Pendulum]{
\begin{minipage}[b]{0.153\linewidth}
\label{pendulum}
\includegraphics[width=1.0\linewidth]{fig/pendulum1.jpg}
\end{minipage}}
\subfigure[Walker2d]{
\begin{minipage}[b]{0.153\linewidth}
\label{pendulum}
\includegraphics[width=1.0\linewidth]{fig/walker.jpg}
\end{minipage}}
\subfigure[LunarLander]{
\begin{minipage}[b]{0.245\linewidth}
\label{pendulum}
\includegraphics[width=1.0\linewidth]{fig/lunar.png}
\end{minipage}}
\subfigure[BipedalWalker]{
\begin{minipage}[b]{0.228\linewidth}
\label{pendulum}
\includegraphics[width=1.0\linewidth]{fig/bipedal.jpg}
\end{minipage}}\\
\caption{Environments used in experiments}
\label{fig_envs}
\vspace{-0.3cm}
\end{figure}

\paragraph{FourRewardMaze} The FourRewardMaze is a 2-D navigation task where an agent need to find all four solutions to achieve better performance. The state space is $2$-D continuous vector indicating the position of the agent, while the action space is a $2$-D continuous value indicating the direction and step length of the agent, which is limited to $[-1,1]$. The initial location of the agent is randomly selected for each game, and each episode has the length of $32$, which is the timesteps needed to collect all four rewards from any starting position. 

\paragraph{Pendulum-v0}
The Pendulum-v0 environment is a classic problem in the control literature. In the Pendulum-v0 of OpenAI Gym. The task has $3$-D state space and $1$-D action space. In every episode the pendulum starts in a random position, and the learning objective is to swing the pendulum up and keep it staying upright.

\paragraph{Walker2d-v2}
The Walker2d-v2 environment is a locomotion task where the learning objective is to make a two-dimensional bipedal robot walk forward as fast as possible. The task has $17$-D state space and $6$-D action space.

\paragraph{LunarLanderContinuous-v2}
In the tasks of LunarLanderContinuous-v2, the agent is asked to control a lander to move from the top of the screen to a landing pad located at coordinate $(0,0)$. The fuel is infinite, so an agent can learn to fly and then land on its first attempt. The state is as $8$-D real-valued vector and action is $2$-D vector in the range of $[-1,1]$, where the first dimension controls main engine, $[-1,0]$ off, $[0.,1]$ throttle from $50\%$ to $100\%$ power and the second value in $[-1,-0.5]$ will fire left engine, while a value in $[0.5,1.0]$ fires right engine, otherwise the engine is off.

\paragraph{BipedalWalker-v3}
The BipedalWalker-v3 is a locomotion task where the state space is $24$-D and the action space is $4$-D. The agent needs to walk as far as possible in each episode where a total timestep of $1000$ are given and total $300$ points might be collected up to the far end. If the robot falls, it gets $-100$ points. Applying motor torque costs a small amount of points, more optimal agent will get better score. 
\section{Reproduction Checklist}
\subsection{Neural Network Structure}
In all experiments, we use the same neural network structure: in TD3, we follow the vanilla implementation to use $3$-layer fully connected neural networks where $256$ hidden units are used. In the selector networks of the INVASE module, we follow the vanilla implementation to use $3$-layer fully connected neural networks where $100$, $200$ hidden units are used.
\subsection{Hyper-Parameters}
In both TD-SWAR and the Dyn-SWAR, we apply IC-INVASE with $p_r$ reducing from $0.5$ to $0.0$ and $\lambda$ increasing from $0.0$ to $0.2$. While our experiments have already shown the effectiveness and robustness of those hyper-parameters, performing grid search on those hyper-parameters may lead to further performance improvement.

\end{document}